%% file: paper.tex
\documentclass[final,12pt]{clear2023} 

\title[Distinguishing Cause from Effect on Categorical Data]{Distinguishing Cause from Effect on Categorical Data: \\
The Uniform Channel Model}

\usepackage{times}

\usepackage{graphicx} 
\usepackage{siunitx}
\usepackage{comment}
\usepackage[american]{babel}
\usepackage{enumitem}

\usepackage[utf8]{inputenc} 
\usepackage[T1]{fontenc}    
\usepackage{hyperref}       
\usepackage{url}            
\usepackage{booktabs}       
\usepackage{amsfonts}       
\usepackage{nicefrac}       
\usepackage{microtype}      
\usepackage{xcolor}         

\usepackage{array}
\usepackage{enumitem}

\usepackage{mathtools} 
\usepackage{tikz} 

\clearauthor{\Name{Mário A. T. Figueiredo,} \Email{mario.figueiredo@tecnico.ulisboa.pt} \\
  \Name{Catarina Oliveira} \Email{catarina.a.oliveira@tecnico.ulisboa.pt}\\
  \addr Instituto de Telecomunicações and LUMLIS (Lisbon ELLIS Unit), \\
  Instituto Superior Técnico, Universidade de Lisboa, Portugal}

\usepackage[normalem]{ulem} 

\begin{document}

\maketitle

\begin{abstract}%
Distinguishing cause from effect using observations of a pair of random variables is a core problem in causal discovery. Most approaches proposed for this task, namely \textit{additive noise models} (ANM), are only adequate for quantitative data. We propose a criterion to address the cause-effect problem with categorical variables (living in sets with no meaningful order), inspired by seeing a conditional \textit{probability mass function} (pmf) as a discrete memoryless channel. We select as the most likely causal direction the one which the conditional pmf is closer to a \textit{uniform channel} (UC). The rationale is that, in a UC, as in an ANM, the conditional entropy (of the effect given the cause) is independent of the cause distribution, in agreement with the principle of \textit{independence of cause and mechanism}. Our approach, which we call the \textit{uniform channel model} (UCM), thus extends the ANM rationale to categorical variables. To assess how \textit{close} a conditional pmf (estimated from data) is to a UC, we use statistical testing, supported by a closed-form estimate of a UC channel. On the theoretical front, we prove identifiability of the UCM and show its equivalence with a structural causal model with a low-cardinality exogenous variable. Finally, the proposed method compares favorably with recent state-of-the-art alternatives in experiments on synthetic, benchmark, and real data.

\end{abstract}


\input{intro}
\input{method}

\input{channel}

\input{criterion}

\input{resul}

\input{concl}



\clearpage
\bibliography{references}

\clearpage
\appendix
\input{appendix_a}
\input{appendix_b}

\section{Detailed Description of the Datasets Used in Section \ref{sec:real_data}}
\label{app_data}
\textbf{Adult} - This dataset consists of 48832 records from the census database of the US in 1994. We consider the following pairs: \textit{(occupation, income)} and \textit{(work class,income)}. The variable \textit{occupation} takes values in the set $\{$\textit{admin, armed-force, blue-collar, white-collar, service, sales, professional}, \textit{other-occupation}$\}$. The variable \textit{work class} takes categories in $\{$\textit{private, self-employed, public servant, unemployed}$\}$. Finally, \textit{income} is a binary variable taking value in $\{ >50,\, \leq 50\}$. Following \citet{paper_stochastic}, we assume that the ground truth is \textit{occupation} $\rightarrow$ \textit{income} and \textit{work class} $\rightarrow$ \textit{income}.

\textbf{Pittsburgh Bridges} - This dataset contains records about 108 bridges and some of their characteristics. We consider the  pairs (\textit{purpose}, \textit{type}) and (\textit{material},\textit{lanes}). The variable \textit{purpose} takes values in $\{$\textit{Walk, Aqueduct, RR, Highway}$\}$; variable \textit{type} takes values in $\{$\textit{Wood, Suspen, Simple-T, Arch, Cantilev, CONT-T}$\}$; variable \textit{material} takes values in $\{$\textit{Steel, Iron, Wood}$\}$; variable \textit{lanes} takes values in $ \{1,2,4,6\}$. Following \citet{hcr}, the ground truths assumed are \textit{material} $\rightarrow$ \textit{lanes} and \textit{purpose} $\rightarrow$ \textit{type}. 

\textbf{Acute Inflammations} - This dataset consists of 120 patients and whether each patient is experiencing a specific symptom, the temperature, and whether he/she suffers from acute inflammations of the urinary bladder and/or acute nephritis. We consider the binary variables \textit{ocurrence of nausea} $(Y_1)$, \textit{lumbar pain} $(Y_2)$, and \textit{burning of urethra} $(Y_3)$ (naturally, taking value 1 if the patient has that symptom, and 0 otherwise). Moreover, $X$ represent the diagnosis \textit{inflammation of urinary bladder}. Following \citet{anm2011}, the goal is to model the diagnosis process, thus we expect $Y_j \rightarrow X$, for $j = 1,...,3$. Notice that the variables $X_i$ only correspond to the diagnosis, not necessarily the truth, otherwise they would be considered the cause, rather than the effect.  

The $\chi^2$ test did not reject the null hypothesis of independence (at a significance level of 5\%) for the following pairs of variables in this dataset: (\textit{Inflammation of urinary bladder}, \textit{Occurrence of nausea}); (\textit{Inflammation of urinary bladder}, \textit{Burning of urethra}); (\textit{Nephritis}, \textit{Micturition pains}). Although, intuitively, we would expected a causal relation between those, the statistical evidence is not strong enough in favour of their mutual dependency (according to the $\chi^2$ test), therefore, these pairs were discarded from the experiments. 

\textbf{Temperature} - This dataset consists of 9162 daily values of temperature measured in Furtwangen (Germany), with the variable \textit{day of the year} taking integer values from 1 to 365 (or 366 for leap years) and \textit{temperature} in $^{\circ} C$. Here, we aggregate days associated with each month and take \textit{month} $\rightarrow$ \textit{temperature} as the ground truth, as \citet{anm2011}. Notice that \textit{month} assumes a cyclic structure.

\textbf{Horse Colic} - This dataset contains 368 medical records of horses. We study the causal relationship between the variable \textit{abdomen status}, which takes the values in $\{$\textit{Normal, Other, Firm feces in the large intestine, Distended small intestine, Distended large intestine}$\}$, and the binary variable \textit{surgical lesion}, indicating whether the lesion was surgical or not. As \citet{horsedata}, we regard \textit{abdomen status} $\rightarrow$ \textit{surgical lesion} as ground truth.




\end{document}

%% file: intro.tex
\section{Introduction}
\label{sec:intro}
Causal inference is a key problem in many areas of science and data analysis \citep{pearlcausality}. In principle, distinguishing statistical dependencies from causal relationships requires interventions \citep{pearlcausality,elementsofcausalinference}. However, intervening is often impossible (\textit{e.g.}, analyzing past data), impractical, or unethical (\textit{e.g.}, forcing people to smoke), which has stimulated much research aimed at inferring causal relationships (\textit{causal discovery}) from purely observational data \citep{Janzing2019, Mooij2016, elementsofcausalinference}, or mixed observational-interventional data \citep{pmlr-v177-faria22a}

There is a vast literature  on methods to learn \textit{directed acyclic graphs} (DAGs) from  data, usually by inferring \textit{conditional independence} (CI) properties among variables \citep{Chickering2002b,Heckerman,Koller_Friedman}. However, without additional assumptions or criteria, those methods cannot distinguish different DAGs entailing the same CIs (in the same \textit{Markov equivalence class} -- MEC). The simplest instance of this problem involves a pair of variables $(X,Y)$: purely statistical methods (\textit{e.g.}, maximum likelihood estimation) cannot recover the causal graph because $X \rightarrow Y$ and $Y \rightarrow X$ constitute a MEC, corresponding to the two possible factorizations of the joint distribution: $p_{X,Y}(x,y) = p_{Y|X}(y|x) \, p_X(x) = p_{X|Y}(x|y)\, p_Y(y)$. Although there are several methods that select particular elements of a MEC, they all include additional assumptions beyond \textit{faithfulness}\footnote{Fiathfulness holds if every conditional independence in the joint probability distribution corresponds to a separation property in the graph \citep{Koller_Friedman,Sadeghi}.} and, for one reason or another they cannot be used to distinguish between $X \rightarrow Y$ and $Y \rightarrow X$, where $X$ and $Y$ are a pair of categorical variables; \textit{e.g.}, they are only applicable to quantitative data \citep{Park2018} or only make sense for more than two variables \citep{Gao2021}.

Without interventions, choosing an element of a MEC requires additional assumptions about the underlying data-generating mechanism. For instance, \textit{additive noise models} (ANM) \citep{anm1,anm3,anm2011,anm4} assume the effect is a function of the cause plus a noise term independent of the cause ($Y = f_Y(X) + N_Y, X \perp \!\!\! \perp N_Y$); if the same doesn't hold in the reverse direction, the model is said to be \textit{identifiable}. ANMs are generically identifiable \citep{anm2011,anm4} in the following sense: if the joint density $p_{X,Y}$ corresponds to an ANM, the conditional density $p_{Y|X}(\cdot|x)$ has identical \textit{shape} for any $x$, simply being \textit{shifted} by $f_Y(x)$, but $p_{X|Y}(\cdot|y)$ typically depends on $y$ in a more complicated way. Under the ANM criterion, if such a  model exists in one direction but not the other, the former is selected as the causal direction.

The ANM can be seen as an instance of the principle of \textit{independence of cause and mechanism} (ICM) \citep{Janzing2010,kolm3}, according to which the cause-effect mechanism (a deterministic function followed by addition of noise, in an ANM) is \textit{independent}\footnote{The term ``independent'' here does not have a probabilistic sense, but a functional sense: changing the distribution of the cause does not affect the causal mechanim, and vice-versa.} of the cause, thus of its distribution. The ICM principle has been exploited using  different tools to define and assess the notion of \textit{independence}: information geometry \citep{Daniusis, kolm3}; algorithmic information theory, namely Kolmogorov complexity (which is not computable, but is approximable \citep{Vitanyi}), by \citet{Janzing_TIT_2010} and \citet{Mian_Marx_Vreeken_2021}; stochastic complexity, via the \textit{minimum description length} (MDL \citep{mdl}) principle \citep{Budhathoki2017,Marx2019,Tagasovska} or the \textit{minimum message length} \citep{Wallace} criterion, by \citet{Stegle}.

Relatively few methods have been proposed to address the cause-effect problem with categorical variables. \citet{anm2011,Peters_AISTATS_2010}  extended ANMs to the discrete case and proved identifiability, but only for variables taking values in a set equipped with a meaningful order, in which an operation similar to addition (a shift) is defined. They consider the rings $\mathbb{Z}$, for variables without cyclic structure, and $\mathbb{Z}/n\mathbb{Z}$ with modulo-$n$ addition, for cyclic variables (\textit{e.g.},  seasons or months of the year), or subsets of these rings. However, purely categorical variables live in sets with no order, thus no meaningful notion of addition or shift, precluding the direct use of ANMs. \citet{anm2011} consider what they call ``structureless'' sets, but only with a particular form of the conditional pmf, not generally applicable. Some of the MDL-based methods mentioned above can be used with categorical variables: \citet{Budhathoki2017} proposed CISC (\textit{causal inference by stochastic complexity}); \cite{hcr} proposed HCR (\textit{hidden compact representation}), based on BIC (\textit{Bayesian information criterion}). \citet{dc} assess mechanism independence via a \textit{distance correlation} (DC) between the cause pmf and the conditional pmf of the effect. \citet{kocaoglu_entropic_2017}  select the causal direction in which the sum of the marginal entropy of the cause with that of the exogenous variable in the corresponding \textit{structural causal model} (SCM) is minimal. Recently, \citet{ni2022} addressed the cause-effect problem for categorical variables by formulating the conditional distribution of the effect given the cause as ordinal regression, with optimal label permutation, and choosing the direction in which this model has the highest likelihood. 

As is standard when focusing on the cause-effect problem, we assume \textit{causal sufficiency} (\textit{i.e.}, absence of unobserved \textit{confounders}), no selection bias, and no feedback \citep{Mooij2016}. We propose a new approach to the cause-effect problem for categorical variables, inspired by viewing the causal mechanism as a communication channel. This view allows extending to the categorical case a key feature of ANMs: the conditional (differential) entropy of the effect given the cause is independent\footnotemark[1] of the distribution of the cause \citep{kolm3}. For categorical variables (\textit{i.e.}, symbols, in channel terminology), a memoryless channel corresponds to the conditional \textit{probability mass function} (pmf) of the output given the input, the \textit{channel matrix} $\boldsymbol \theta^{X\rightarrow Y}$, where  $\boldsymbol \theta_{x,y}^{X\rightarrow Y} = p_{Y|X}(y|x) = \mathbb{P}[Y=y | X=x]$. In a so-called \textit{uniform channel} (UC \citep{Hamming}), the rows of this matrix are permutations of each other\footnote{A uniform channel is not necessarily a \textit{symmetric channel}, which requires additionally that all the columns are also permutations of each other \citep{code}.}, implying (as shown below) that the conditional entropy $H(Y|X)$ is independent of the distribution of $X$. Paralleling the ANM rationale, given a pair of categorical variables ($X,Y$), if the conditional pmf in one direction, say of $Y$ given $X$, corresponds to a UC and the same is not true in the other direction, then the causal structure is declared to be $X\rightarrow Y$. This criterion, which we refer to as the UCM (\textit{uniform channel model}) is supported by an \textit{identifiability} result proved in this paper: if a joint distribution corresponds to a UCM in one direction, in general (\textit{i.e.}, with probability one under any continuous density on the model parameters), it does not correspond to a UCM in the reverse direction. 

The proposed UCM approach is further supported by the fact (proved below) that if, and only if, $\boldsymbol \theta^{X\rightarrow Y}$ corresponds to a UCM, is it possible to write a \textit{structural causal model} (SCM) \citep{pearlcausality} of the form $Y = f_Y(X,U_Y)$, where $f_Y$ is a deterministic function and $U_Y$ is an exogenous random variable, taking values in in the same set as $Y$ and independent of $X$. The importance of this independence was recently highlighted by \citet{Papineau}: ``the probabilistic independence of exogenous terms in (...) structural equations holds the key to causal direction''. 

A final question is how to instantiate the UCM principle with a finite amount of data. This question parallels that of how to estimate the underlying function and noise distribution in an ANM. Naturally, with a finite dataset, we only have an estimate of the underlying distribution and the probability that this estimate corresponds exactly to a UCM in one of the two directions is vanishingly small. Although other ways to address this issue are conceivable, we resort to statistical hypothesis testing to decide in which direction, if any, the conditional pmf can be considered a UCM. A key building block of this approach is estimating a channel under the constraint that it is uniform; this is a problem that, to the best of our knowledge, had not be studied before and for which we derive a closed-form solution. We also extend the approach to the case where the rows of the channel matrix are cyclic permutations of each other (a \textit{cyclic UCM} -- CUCM), applicable when the effect variable has cyclic nature, but in this case the channel estimate has to be obtained iteratively.

It is important to stress that using the UCM (or ANM, or any other restricted model class, for that matter) to identify a causal relation does not imply any assumption that this is a realistic model of the true underlying relation. As clearly argued by \citet[Section 4.1.2]{elementsofcausalinference}, the rationale is simply that if there is such a model in one direction, but not the other, it is more likely that the former is the causal direction. 

The main contributions of this paper are the following:

\begin{itemize}
    \item A new instantiation, for categorical variables, of the principle of \textit{independence of cause and mechanism}: the \textit{uniform channel model} (UCM) principle.
    \item A proof of identifiability of the UCM. 
    \item A proof that the joint distribution of a pair of categorical random variables is entailed by an SCM in which the exogenous noise has the same cardinality as the effect variable if and only if it corresponds to a UCM.
    \item An instantiation of the UCM principle using statistical hypothesis testing, supported on a closed-form estimate of a UCM (which, to the best of our knowledge, is a new result, possibly of independent interest).
\end{itemize}

The paper is organized as follows. Section \ref{sec:UCModel} describes the UCM and presents the corresponding identifiability theorem and equivalence to an SCM. Section \ref{sec:channel} addresses the problem of estimating uniform and cyclic uniform channels from  data. Section \ref{sec:criterion_data} 
describes how the criterion is applied to observed data. Experimental results are reported in Section \ref{sec:resul}, and Section \ref{sec:concl} concludes the paper.

%% file: method.tex
\section{Uniform Channel Models -- UCM -- for Categorical Variables}
\label{sec:UCModel}
This section describes the proposed causal inference principle for categorical variables, after introducing notation and reviewing the notion of uniform channel. Finally, we prove an identifiability theorem for the proposed model and show its equivalence with an SCM. 

\label{sec:method}
\subsection{Categorical Variables and Uniform Channels}
Let $X \in \mathcal{X} = \{1,...,|\mathcal{X}|\}$ and $Y \in \mathcal{Y} = \{1,...,|\mathcal{Y}|\}$ be two \textit{categorical random variables} (although their outcomes are shown as integers, no role is played by their order). The joint pmf $p_{X,Y}$ can be factored in two different ways, $
 p_{X,Y}(x,y) = \; p_{Y|X}(y|x) \; p_X(x)  =  p_{X|Y}(x|y)\; p_Y(y)$,  corresponding to a \textit{Markov equivalence class}. If $X  \perp \!\!\! \perp Y$, the joint pmf factors trivially $p_{X,Y}(x,y) = \; p_X(x)\, p_{Y}(y)$. Let the vector of parameters of the first factorization be denoted as $\boldsymbol \theta = (\boldsymbol \theta^X, \boldsymbol \theta^{X\rightarrow Y})$, \textit{i.e.},
\begin{align*}
    \theta_x^X = p_X(x) = \mathbb{P}[X=x] \quad \text{and} \quad  \theta_{x,y}^{X\rightarrow Y} = p_{Y|X}(y|x) = \mathbb{P}[Y=y | X = x],
\end{align*}
where $\boldsymbol \theta^X \in \Delta_{|\mathcal{X}| - 1}$, with $\Delta_{m-1}$ being the probability simplex in $\mathbb{R}^m$. The conditional probabilities are arranged in a  $|\mathcal{X}| \times |\mathcal{Y}|$ row-stochastic matrix $\boldsymbol \theta^{X\rightarrow Y}$, with the  $x$-th row denoted as $\boldsymbol \theta_x^{X\rightarrow Y}$. 

\vspace{0.05cm}

\begin{definition}[Discrete Memoryless Channel] \citep{code,Hamming}
A discrete memoryless channel (DMC) is a probabilistic system with a discrete input alphabet $\mathcal{X} = \{1,\ldots, |\mathcal{X}|\}$ and a discrete output alphabet $\mathcal{Y}= \{1,\ldots, |\mathcal{Y}|\}$, specified by the conditional probabilities $p_{Y|X}( y| x)$, for $x \in \mathcal{X}$ and $y \in \mathcal{Y}$.  The adjective ``memoryless" means that, given a sequence of random inputs, the corresponding outputs are conditionally independent.

\end{definition}

\begin{definition}[Uniform channel (UC)] \citep{Hamming}
A UC is a DMC in which each row of the conditional probability (channel) matrix $\boldsymbol \theta^{X\rightarrow Y}$ is a permutation of every other row. 
\end{definition}

\begin{definition}[Cyclic uniform channel (CUC)] A CUC is a UC where each row of the channel matrix $\boldsymbol \theta^{X\rightarrow Y}$ is a cyclic permutation of every other row.
\end{definition}

Let $\mathbb{S}_{|\mathcal{Y}|}$ denote the set of all $|\mathcal{Y}|!$ permutations of $(1,..., |\mathcal{Y}|)$. In a UC, each row $\boldsymbol \theta_x^{X\rightarrow Y}$ is a row-specific permutation $\sigma_x \in \mathbb{S}_{|\mathcal{Y}|}$ of a common vector $\boldsymbol\gamma \in \Delta_{|\mathcal{Y}|-1}$, \textit{i.e.},
\begin{equation}
\boldsymbol \theta_x^{X\rightarrow Y} = (\gamma_{\sigma_x (1)}, \ldots,   \gamma_{\sigma_x (|\mathcal{Y}|)}) \;\;\; \Leftrightarrow \;\;\; p_{Y|X}(y|x) = \gamma_{\sigma_x(y)}. \label{eq_gama_perm}
\end{equation}
In the case of a CUC, $\sigma_x \in \mathbb{C}_{|\mathcal{Y}|}$, the set of all $|\mathcal{Y}|$ cyclic permutations of $(1,..., |\mathcal{Y}|)$.

\subsection{Uniform Channel Model -- UCM -- for Categorical Variables}
We propose a new principle to infer the most likely causal direction between two categorical variables by following the rationale behind ANMs. Recall that the ANM principle for real variables is as follows: if $Y$ satisfies an ANM $Y= f_Y(X) + N_Y$, where $N_Y \perp \!\!\! \perp X$ (\textit{i.e.}, the \textit{exogenous} noise is independent of $X$), but the same is not true in the reverse direction, then the most likely causal direction is $X\rightarrow Y$. In the ANM for real variables, the conditional \textit{probability density function} (pdf) has identical \textit{shape} for all values of $x$, simply being \textit{shifted} by $f_Y(x)$, \textit{i.e.}, $p_{Y|X}(y|x) = p_{N_Y} (y-f_Y(x))$, where $p_{N_Y}$ is the pdf of the noise variable $N_Y$. Consequently, the \textit{conditional differential entropy} $h(Y|X)$ does not depend on the pfd of $X$, as shown next. 

\begin{proposition}\label{conditioned_h}
If real-valued variables $X$ and $Y$ admit an ANM from $X$ to $Y$, then the \textit{conditional differential entropy} $h(Y|X) = h(N_Y)$, independently of the distribution of $X$.
\end{proposition}

{\noindent\bf Proof}: Using the shift-invariance  (a) of differential entropy \citep{code},
\begin{align}
h(Y|X) & =  \; \mathbb{E}_{X,Y} [ -\log p_{Y|X}(Y|X) ] \; = \; \mathbb{E}_{X} \bigl[  \mathbb{E}_{Y|X}[ -\log p_{Y|X}(Y|X) ] \bigr] \nonumber \\ 
& = \;  \mathbb{E}_{X} \bigl[  \mathbb{E}_{Y|X}[ -\log p_{N_Y}(Y-f_Y(X)) ] \bigr] \; \stackrel{(a)}{=}\; \mathbb{E}_{X} \bigl[ h(N_Y) \bigr]  \;  = \; h(N_Y).  \tag*{\mbox{$\blacksquare$}}
\end{align}

For categorical variables, the sets $\mathcal{X}$ and $\mathcal{Y}$ lack any meaningful order, thus there is no notion of \textit{addition}, and an ANM is not directly applicable. However, the conditional entropy invariance property of ANMs can be preserved by considering the transformation group under which discrete entropy is invariant: \textit{permutations}. Consequently, our proposed causal inference principle is:

\begin{center}
\fbox{
\begin{minipage}{13.5cm}
\begin{description}[leftmargin=*]
\item[UCM causal inference principle for categorical variables:] given two categorical variables $X$ and $Y$, if the conditional pmf $\boldsymbol \theta^{X \rightarrow Y}$ corresponds to a UCM, but the conditional pmf $\boldsymbol \theta^{Y\rightarrow X}$ does not, then we infer the causal direction to be $X\rightarrow Y$.
\end{description}
\end{minipage}}
\end{center}

Paralleling Proposition \ref{conditioned_h}, the following result is a simple consequence of the invariance of (discrete) entropy to symbol permutations. 

\begin{proposition}\label{conditioned_H}
If $\boldsymbol \theta^{X\rightarrow Y}$ corresponds to a UC (each row of $\boldsymbol \theta^{X\rightarrow Y}$ is a permutation of a vector $\boldsymbol\gamma \in \Delta_{|\mathcal{Y}|-1}$), then the conditional entropy $H(Y|X) = H({\boldsymbol \gamma})$, independently of $p_X$.
\end{proposition}

{\noindent\bf Proof}: Due to the permutation-invariance property (a) of entropy \citep{code},
\begin{align}
H(Y|X) & = \mathbb{E}_{X,Y} [ -\log p_{Y|X}(Y|X) ] \; = \; \mathbb{E}_{X} \bigl[  \mathbb{E}_{Y|X}[ -\log p_{Y|X}(Y|X) ] \bigr] \nonumber \\
& = \; \mathbb{E}_{X} \bigl[  \mathbb{E}_{Y|X}[ -\log \gamma_{\sigma_X(Y)} ] \bigr] \;  
 \stackrel{(a)}{=} \; \mathbb{E}_{X} \bigl[ H({\boldsymbol\gamma}) \bigr] 
\; = \; H({\boldsymbol\gamma}).  \tag*{\mbox{$\blacksquare$}}
\end{align}

The proposed causal inference problem can be seen as an instance of the \textit{independence of cause and mechanism} principle,  with \textit{independence} corresponding to the following property: the conditional pmf of the effect given the cause has the same collection of probability values, only their positions depend on the cause. Thus, the conditional uncertainty (entropy) of the effect, given the cause, is independent of the distribution of the cause. 

A relevant fact that provides further support to the proposed principle is that the UCM can be written as an SCM \citep{pearlcausality}, as shown in the following proposition.
\begin{proposition}
Let $X\in\mathcal{X}$ and $Y\in\mathcal{Y}$ be a pair of dependent random variables such that the conditional pmf $\boldsymbol \theta^{X\rightarrow Y}$ corresponds to a UCM specified by $\boldsymbol\gamma$ and $\sigma_1,...,\sigma_{|\mathcal{X}|}$, as in \eqref{eq_gama_perm}, and the marginal $\boldsymbol \theta^{X}$ be arbitrary. Then, the joint pmf of $X$ and $Y$ is entailed by the following SCM:
\begin{equation}
X:= U_X, \hspace{1cm} Y := f_Y( X, U_Y), \label{eq:SCM}
\end{equation}
with $U_X \perp \!\!\! \perp U_Y$ (independent \textit{exogenous}  variables), $U_X \in \mathcal{X}$ has pmf $\boldsymbol \theta^{X}$, $U_Y \in \mathcal{Y}$ has pmf $\boldsymbol{\gamma}$, and $f_Y:\mathcal{X} \times \mathcal{Y} \rightarrow \mathcal{Y}$ is a  function given by $ f_Y(x,u) = \tau_x (u)$, with $\tau_x = \sigma_x^{-1}$ (inverse\footnote{Given a permutation $\sigma \in \mathbb{S}_{|\mathcal{Y}|}$, its inverse $\sigma^{-1} \in \mathbb{S}_{|\mathcal{Y}|}$ is such that $\sigma^{-1} (\sigma (i)) = i$, for any $i = 1,..., |\mathcal{Y}|$.} permutation of $\sigma_x$). Conversely, if the conditional pmf $\boldsymbol \theta^{X\rightarrow Y}$ does not correspond to a UCM, it is  impossible to write an SCM of the form \eqref{eq:SCM}, with $U_Y\in \mathcal{Y}$, entailing the joint pmf of $X$ and $Y$.

\end{proposition}
{\noindent\bf Proof}: 
Given $x\in\mathcal{X}$, $Y = f_Y(x,U)$ is a categorical random variable with conditional pmf
\[
p_{Y|X}(y|x) = \mathbb{P}[Y=y | X=x] = \mathbb{P}[ U_Y = \tau_x^{-1}(y) ] = \gamma_{\sigma_x(y)}
\]
(which coincides with \eqref{eq_gama_perm}), where the second equality stems from the definition of $f_Y$ and permutations being bijections. Conversely, if the conditional pmf $\boldsymbol{\theta}^{X\rightarrow Y}$ does not correspond to a UCM (neither all its rows are equal to each other, by assumption), depending on what value/category $X$ takes, the conditional pmf of $Y$ takes different probability values, not just a permutation of a common pmf, making it impossible to write an SCM of the form \eqref{eq:SCM} with $U_Y \in \mathcal{Y}$ independent of $U_X$. \hfill$\blacksquare$

\vspace{0.2cm}
It is the restriction $U_Y\in \mathcal{Y}$ that makes this result non-trivial. In fact, \citet{kocaoglu_entropic_2017} showed that, given any joint pmf $p_{X,Y}$, it is possible to write an SCM $Y := f_Y( X, U_Y)$, with $U_X \perp \!\!\! \perp U_Y$, that induces $p_{X,Y}$ and such that $U_Y$ takes values in a set of cardinality O($|\mathcal{X}|\, |\mathcal{Y}|$). 

\begin{remark}
If none, or both, of the conditional pmfs, $\boldsymbol \theta^{X\rightarrow Y}$ and $\boldsymbol \theta^{Y\rightarrow X}$, correspond to a UCM, the proposed criterion does not select a causal direction. Of course, in practice, the conditional pmfs are estimated from a finite dataset, thus there is a very small chance that one of these estimates corresponds exactly to a UCM. In Section \ref{sec:criterion_data}, we come back to this issue, proposing statistical tests do decide if a conditional pmf estimate can be considered to correspond to a UCM. In the following subsection, we assume that we have an infinite amount of data (equivalently, the true underlying pmf) and address the identifiability issue in this ideal condition. If this model was not identifiable in this ideal setting, it would be hard to argue that it could be useful with a finite amount of data.
\end{remark}

\begin{remark}
Our UCM contains as a particular case the model for ``structureless" sets proposed by \citet{anm2011}. Their model assumes a function $\phi:\mathcal{X}\rightarrow\mathcal{Y}$ and  $p_{Y|X}(y|x) = p$, if $y=\phi(x)$, and $p_{Y|X}(y|x) = (1-p)/(|\mathcal{Y}|-1)$, if $y\neq \phi(x)$. This corresponds to a UC (in fact, a CUC), with
\[
\boldsymbol\gamma = \bigl( p,(1-p)/(|\mathcal{Y}|-1),...,(1-p)/(|\mathcal{Y}|-1) \bigr)
\]
and any set of permutations such that $\sigma_x(y) = 1$, for $y=f(x)$. Our UC and CUC models are much more general, as they do not constrain the conditional pmf to have only two different values.
\end{remark}

\subsection{Identifiability}\label{Indentify}
For the proposed criterion to be useful, it should be supported by an identifiability guarantee, \textit{ i.e.}, that the set of joint probability mass functions $p_{X,Y}$ such that both $p_{Y|X}$ and $p_{X|Y}$ correspond to uniform channels should be as small as possible, ideally have zero Lebesgue measure in the space of valid parameters, thus zero probability under any continuous density \citep{anm2011}.  Before stating and proving the general identifiability result, we illustrate it for the case where both variables are binary: $\mathcal{X} = \mathcal{Y}= \{1, 2\}$. Let $p_X(1) = \theta_1^X = \beta$ and let $\boldsymbol \theta^{X\rightarrow Y}$ correspond to a UCM (in this case, simply a \textit{binary symmetric channel}) with error probability $\alpha$ \citep{code}:
\begin{align*}
    \boldsymbol \theta^{X\rightarrow Y} = \begin{bmatrix}
    1 - \alpha &  \alpha\\
    \alpha & 1 - \alpha
    \end{bmatrix}.
\end{align*}
Of course, a channel matrix where the two rows are equal to $(1-\alpha,\alpha)$ is also a UC, but in that case $X$ and $Y$ are independent, which is an uninteresting case. The channel in the reverse direction, \textit{i.e.}, $\boldsymbol \theta^{Y\rightarrow X}$, can be easily derived using Bayes law, yielding
\begin{align*}
{\small    \boldsymbol \theta^{Y\rightarrow X} = \begin{bmatrix}
    {\displaystyle \frac{(1-\alpha) \beta}{(1-\alpha) \beta + \alpha (1-\beta)}} &{\displaystyle  \frac{\alpha(1 - \beta)}{(1-\alpha) \beta + \alpha(1-\beta)}} \\
 {\displaystyle \frac{\alpha\beta}{\alpha\beta + (1-\alpha) (1 - \beta)} } & {\displaystyle \frac{(1-\alpha) (1 - \beta)}{\alpha\beta + (1-\alpha) (1 - \beta)}}
    \end{bmatrix} .}
\end{align*}
Notice that in matrix $\boldsymbol \theta^{Y\rightarrow X}$, the variable $Y$ indexes rows and $X$ indexes columns, so that it is  row-stochastic as is standard for channel matrices. Matrix $\boldsymbol\theta^{Y\rightarrow X}$ represents a UC if and only if 
one (or both) of two conditions are satisfied: the diagonal elements are equal to each other; the elements in the first column are equal to each other (in which case, $X \perp\!\!\! \perp Y$). Simple algebraic manipulation allows showing that this is equivalent to having $(\alpha,\beta) \in \{(\alpha,\beta)\in[0,1]^2: \, \alpha = 0 \, \vee\,  \alpha = 1/2 \, \vee\, \alpha = 1 \, \vee\, \beta = 0 \, \vee\, \beta = 1/2 \, \vee\, \beta = 1\}$, which has zero Lebesgue measure. The following theorem generalizes this result for arbitrary $|\mathcal{X}|$ and $|\mathcal{Y}|$.

 \begin{theorem}\label{identifiability}
Let $X \in \mathcal{X}$ and $Y\in \mathcal{Y}$ be two categorical random variables with a joint pmf such that the conditional $\boldsymbol \theta^{X\rightarrow Y}$ corresponds to a UC. Assume also that the marginals have full support\footnote{There is no loss of generality in this assumption; if there are zeros in the marginals, we simply redefine $\mathcal{X}$ or/and $\mathcal{Y}$ by removing the zero-probability elements.}: $p_Y(y) \neq 0$, for any $y\in\mathcal{Y}$, and $p_X(x) \neq 0$, for any $x\in\mathcal{X}$. Further assume that the rows of the channel matrix $\boldsymbol \theta^{X\rightarrow Y}$ are not all equal to each other (\textit{i.e.}, $X$ and $Y$ are not independent\footnote{If all the rows are equal to each other, then $Y \perp \!\!\! \perp X$; since independence is a symmetrical relationship, the reverse channel $\boldsymbol \theta^{Y\rightarrow X}$ will also have all its rows equal to each other, thus being a special case of a UC channel.}).
Then, the set of parameters such that the reverse channel $\boldsymbol \theta^{Y\rightarrow X}$ is also a UCM has zero Lebesgue measure.
\end{theorem}

The proof, presented in Appendix \ref{app_proof}, essentially boils down to showing that the UC condition on the reverse channel  $\boldsymbol \theta^{Y\rightarrow X}$ corresponds to the zero set of a polynomial that is not identically zero, which thus has zero Lebesgue measure, a classical result from polynomial theory \citep{Federer1969}.

%% file: channel.tex
\section{Channel Estimation from Data}
\label{sec:channel}
This section addresses the problem of estimating channel parameters $\boldsymbol \theta^{X\rightarrow Y}$ from $N$ independent and identically distributed samples of $(X,Y): (x_1,y_1),..., (x_N, y_N)$. This will play a key role in translating the causal inference criterion proposed in Section \ref{sec:UCModel} to the realistic scenario where there is only access to a finite amount of data, rather than perfect knowledge of the joint pmf $p_{X,Y}$.

Before estimating the channel parameters (conditional pmf), notice that estimating the marginal pmf $\boldsymbol \theta^{X}$ is trivial: with $N_x$ denoting the number of samples $(x_i, y_i)$ with $x_i = x$, the \textit{maximum likelihood} (ML) estimate of $\boldsymbol \theta^X$ is given by 
\begin{equation}
    \hat{\boldsymbol \theta}^{X} = \; \arg \max_{{\boldsymbol \theta} \in \Delta _{|\mathcal{X}|-1}} \sum_{x\in\mathcal{X}} N_x \log  \theta^X_{x} = \Bigl( \frac{N_1}{N},...,\frac{N_{|\mathcal{X}|}}{N}\Bigr).
    \label{eq:ML_thetax}
\end{equation}

For estimating $\boldsymbol \theta^{X\rightarrow Y}$, we consider the following 4 scenarios:  \textbf{(1)} arbitrary channel; \textbf{(2)} UCM, with known permutations; \textbf{(3)} UCM,  with \textit{unknown} permutations; \textbf{(4)} CUCM, with \textit{unknown} cyclic permutations. Scenarios 1 and 2 are trivial and considered only as they provide the building blocks to address scenarios 3 and 4.

\subsection{Scenario 1: Arbitrary Channel}
\label{subsec: firstscenario}
Let $N_{x,y}$ be the number of samples $(x_i, y_i)$ such that $x_i = x$ and $y_i=y$. In the absence of constraints other than each row of $\boldsymbol \theta^{X\rightarrow Y}$ must be a valid pmf, the ML estimate is 

\begin{equation}
    \hat{\boldsymbol{\theta}}^{X\rightarrow Y}  =  \underset{\boldsymbol{\theta} \in (\Delta_{|\mathcal{Y}|-1})^{|\mathcal{X}|}}{\text{arg max}} \sum_{x\in \mathcal{X}} \sum_{y\in \mathcal{Y}}  N_{x,y}\log \theta_{x,y}.
    \label{objectivefunction}
\end{equation}
Since both the objective function and the constraints in \eqref{objectivefunction} are separable across $x= 1,..., |\mathcal{X}|$, the problem is also separable into a collection of problems, each yielding the classical ML estimates
\begin{equation}
    \hat{\theta}^{X\rightarrow Y}_{x,y} = N_{x,y}/N_x, \hspace{0.3cm} \mbox{for $(x,y)\in \mathcal{X}\times \mathcal{Y}$}.
\label{estimate_thetaxy}
\end{equation}

\subsection{Scenario 2: UCM with Known Permutations}
\label{2ndscenario}
In a UC, each row $\boldsymbol \theta_x^{X\rightarrow Y}$ is as given in \eqref{eq_gama_perm}. If the permutations $\sigma_1, \ldots, \sigma_{|\mathcal{X}|}$ are known, the ML estimate of $\boldsymbol \gamma$ is given by 
\begin{equation}
    \hat{\boldsymbol{\gamma}} = \;  \underset{\boldsymbol{\gamma} \in \Delta _{|\mathcal{Y}|-1} }{\text{arg max}} \sum_{x \in \mathcal{X}} \sum_{y\in \mathcal{Y}}  N_{x,y}\log \gamma_{\sigma_x (y)}.
    \label{objectivefunction_symmetric}
\end{equation}
This problem is not separable, as all the rows of the channel matrix share the same probability values, although with different permutations. Swapping the summation order  and using the inverse permutations $\tau_x = \sigma_x^{-1}$ to do a change of variable in the sum over $\mathcal{Y}$, problem \eqref{objectivefunction_symmetric}  can be rewritten as
\begin{equation}
    \hat{\boldsymbol{\gamma}} = \; \underset{\boldsymbol{\gamma} \in \Delta _{|\mathcal{Y}|-1} }{\text{arg max}}  \sum\limits_{z\in \mathcal{Y}} \log \gamma_z \sum_{x\in \mathcal{X}} N_{x,\tau_{x}(z)}.
    \label{eq:UDC_known}
\end{equation}
Problem \eqref{eq:UDC_known} has the same form as \eqref{eq:ML_thetax}, the solution being simply
\begin{equation}
    {\displaystyle \hat{\gamma}_y =  \frac{1}{N}\sum\limits_{x\in \mathcal{X}} N_{x, \tau_x (y)}}, \;\; \; \mbox{for $y \in \mathcal{Y}$}.
\label{approximategamma}
\end{equation}
Notice that $N_{x, \tau_x (y)}$ is the number of samples $(x_i,y_i)$ such that $x_i=x$ and $y_i = \tau_x (y)$.

\subsection{Scenario 3: UC with Unknown Permutations}
\label{subsec: thirdscenario}
In this case, the log-likelihood is maximized, not only w.r.t. $\boldsymbol\gamma$, but also the permutations. Although, at first sight, this may look like a very hard problem, as there are $(|\mathcal{Y}|!)^{|\mathcal{X}|}$ combinations of permutations, we show next that it can be solved very efficiently. The optimization problem in hand (formulated w.r.t. the inverse permutations, denoted as $\tau_1,\ldots, \tau_{|\mathcal{X}|} \in \mathbb{S}_{|\mathcal{Y}|}$) is
\begin{align}
    \hat{\boldsymbol \gamma}, \hat{\tau}_1, \ldots, \hat{\tau}_{|\mathcal{X}|} \;  = \hspace{-0.5cm} \underset{ \begin{array}{c} \boldsymbol \gamma \in \Delta _{|\mathcal{Y}|-1} \\ 
    \tau_1,...,\tau_{|\mathcal{X}|} \in \mathbb{S}_{|\mathcal{Y}|}\end{array} }{\text{arg max}} \hspace{-0.5cm} \mathcal{L}(\boldsymbol \gamma, \tau_1,...,\tau_{|\mathcal{X}|} ), \label{thirdscenario}
\end{align}
where 
\begin{align}
    \mathcal{L}(\boldsymbol \gamma, \tau_1,...,\tau_{|\mathcal{X}|} ) = \sum_{x\in \mathcal{X}} \sum_{y\in \mathcal{Y}}  N_{x,\tau_x (y)}\log \gamma_y. \label{thirdscenario_b}
\end{align}
The following proposition (proved in Appendix \ref{proof_estimate}) provides the solution to this problem.

\begin{proposition}\label{proof_UCestimate}
A globally optimal solution to the problem specified in \eqref{thirdscenario}--\eqref{thirdscenario_b} is given as follows. For $x\in \mathcal{X}$, $\hat{\tau}_x$ is any permutation that sorts $\{N_{x,1},..., N_{x,|\mathcal{Y}|}\}$ into non-increasing order,
\begin{equation}
\hat{\tau}_x\;\; \mbox{is such that}\;\; N_{x, \hat{\tau}_x (1)} \geq \cdots \geq N_{x, \hat{\tau}_x (|\mathcal{Y}|)},\label{sortedNx_a}
\end{equation}
and, for  $y \in \mathcal{Y}$,
\begin{equation}
    {\displaystyle \hat{\gamma}_y = \frac{1}{N}\sum\limits_{x\in \mathcal{X}} N_{x, \hat\tau_x (y)}}.
\label{approximategamma2_a}
\end{equation}
\end{proposition}

The solution in \eqref{sortedNx_a}--\eqref{approximategamma2_a} is a global, but not unique, optimum; in fact, any pmf $\hat{\boldsymbol{\xi}}$ that is a permutation of $\hat{\boldsymbol{\gamma}}$, \textit{i.e.}, $\hat{\gamma}_y = \hat{\xi}_{\rho(y)}$, where $\rho \in \mathbb{S}_{|\mathcal{Y}|}$, yields 
\[
\hat{\theta}_{x,y}^{X\rightarrow Y} = \hat{\gamma}_{\hat{\sigma}_x(y)} = \hat{\xi}_{\rho(\hat{\sigma}_x(y))}.
\]
That is, $\boldsymbol\gamma$ is identifiable only up to a permutation, since any permutation of $\hat{\boldsymbol{\gamma}}$, combined with the inverse of that permutation composed with each $\sigma_x$, yields the same conditional pmf estimate $\hat{\boldsymbol{\theta}}^{X\rightarrow Y}\!\!\!$, thus the same maximum value of the log-likelihood. Finally, notice that the cost of computing this solution scales as $O(|\mathcal{X}|\, |\mathcal{Y}|\, \log |\mathcal{Y}|)$, due to the number $|\mathcal{X}|$ of sorting operations, each of size $|\mathcal{Y}|$.

\subsection{Scenario 4: CUC with Unknown Permutations}
\label{subsec:fourth}
The difference between this and the previous case is that the permutations are now cyclic. Thus, the corresponding optimization problem is identical to \eqref{thirdscenario}--\eqref{thirdscenario_b}, but with the constraint $\tau_1,...,\tau_{|\mathcal{X}|} \in \mathbb{S}_{|\mathcal{Y}|}$ replaced with $\tau_1,...,\tau_{|\mathcal{X}|} \in \mathbb{C}_{|\mathcal{Y}|}$, where $\mathbb{C}_{|\mathcal{Y}|}$ is the set of cyclic permutations of $\{1,...,|\mathcal{Y}|\}$. 

Although the cardinality of $\mathbb{C}_{|\mathcal{Y}|}$ is $|\mathcal{Y}|$, much smaller than that of $\mathbb{S}_{|\mathcal{Y}|}$, which is $|\mathcal{Y}|!$, this problem is harder than  \eqref{thirdscenario}--\eqref{thirdscenario_b}. Whereas the cost of the exact solution of \eqref{thirdscenario}--\eqref{thirdscenario_b} scales with $O(|\mathcal{X}|\, |\mathcal{Y}|\, \log |\mathcal{Y}|)$, exactly solving this problem by exhaustive search costs $O(|\mathcal{Y}|^{|\mathcal{X}|})$. It happens that this problem is a variant of a class of problems known as \textit{multireference alignment}, which is known to be NP-hard \citep{Bandeira}. Exact solutions are thus out of the question for large problems. Here, we propose an alternating maximization approach with two steps:
\begin{itemize}[leftmargin=0.8cm,itemsep=0cm]
    \item Given the current permutation estimates $\hat{\tau}_1, ..., \hat{\tau}_{|\mathcal{X}|}$, update $\hat{\boldsymbol{\gamma}}$ according to \eqref{approximategamma}, with $\tau_x = \hat{\tau}_x$. 
    \item Given the current $\hat{\boldsymbol \gamma}$, maximize w.r.t. the permutations, which is separable across $\tau_1,...,\tau_{|\mathcal{X}|}$:
    \begin{equation}
    \hat{\tau}_x = \underset{\tau \in \mathbb{C}_{|\mathcal{Y}|}}{\text{arg max}} \sum\limits_{y=1}^{|\mathcal{Y}|} N_{x,\tau(y)} \log \hat{\gamma}_y, \;\;\mbox{for $x=1,...,|{\mathcal X}|$.}
    \label{eq: cyclic}
\end{equation}
    This maximization is carried out exactly by considering all the $|\mathcal{Y}|$ cyclic permutations. 
    \end{itemize}
The costs of both steps of this algorithm scale as $O (|\mathcal{X}|\, |\mathcal{Y}|)$. Convergence can be proved via the same approach that is used to prove convergence of the $K$-means algorithm \citep{kmeans}, since both algorithms share a common structure: alternate between exact maximization with respect to real quantities (cluster centers, in K-means, $\boldsymbol{\gamma}$ in our algorithms) and an exact combinatorial optimization (the point-to-cluster assignments in K-means, the cyclic permutations in the proposed algorithms).

%% file: criterion.tex
\section{Applying the UCM Principle from Data}
\label{sec:criterion_data}
Applying the proposed causal inference principle amounts to performing  hypothesis testing concerning the UC nature of the conditional pmf estimates $\hat{\boldsymbol\theta}^{X\rightarrow Y}$ and $\hat{\boldsymbol\theta}^{Y\rightarrow X}$. This is closely related to classical tests for two-way contingency tables \citep{Agresti, Read}. Given a table of counts $N_{x,y}$, let the \textit{null hypothesis} $H_0$ be that these counts can be explained by a UCM in the $X\rightarrow Y$ direction. To test this hypothesis, consider the corresponding maximum log-likelihood (noting that $p_{X,Y}(x,y) = \mathbb{P}[X=x,Y=y] = \theta_x^{X} \, \gamma_{\sigma_x(y)}$, for a UCM),
\begin{equation}
\mathcal{L}_{H_0} = \sum\limits_{x \in \mathcal{X}}\sum\limits_{y \in \mathcal{Y}} N_{x,y} \log\bigl( \hat{\theta}_x^{X} \, \hat{\gamma}_{\hat{\sigma}_x(y)}\bigr) = \sum\limits_{x \in \mathcal{X}} N_x \log  \hat{\theta}_x^{X} + \sum\limits_{x \in \mathcal{X}}\sum\limits_{y \in \mathcal{Y}} N_{x,y} \log\bigl( \hat{\gamma}_{\hat{\sigma}_x(y)}\bigr) ,
\end{equation}
where $\hat{\boldsymbol\theta}^X$, $\hat{\sigma_1},...,\hat{\sigma}_{|\mathcal{X}|}$, and $\hat{\boldsymbol\gamma}$ are the ML estimates obtained as shown in Section \ref{sec:channel}.

The alternative hypothesis is that the channel is arbitrary, with maximum log-likelihood 
\begin{equation}
\mathcal{L}_{\bar{H}_0} = \sum\limits_{x \in \mathcal{X}} N_x \log  \hat{\theta}_x^{X} + \sum\limits_{x \in \mathcal{X}}\sum\limits_{y \in \mathcal{Y}} N_{x,y} \log\bigl( N_{x,y}/N_x\bigr),
\end{equation}
since the ML estimates of the conditional pmf parameters are as given in \eqref{estimate_thetaxy}, and the ML estimate of the marginal $\boldsymbol\theta^X$  is the same, regardless of the channel being uniform or not. These models are nested: a UCM is a particular case of the set of all valid channels, thus it is always true that $\mathcal{L}_{H_0} \leq \mathcal{L}_{\bar{H}_0}$.

The \textit{likelihood-ratio statistic} (LRS), denoted $G^2$, is then given by 
\begin{equation}
G_{X\rightarrow Y}^2 = 2 (\mathcal{L}_{\bar{H}_0} - \mathcal{L}_{H_0}) =  2 \sum\limits_{x \in \mathcal{X}}\sum\limits_{y \in \mathcal{Y}} N_{x,y} \log\Bigl( \frac{N_{x,y}}{\hat{\gamma}_{\hat{\sigma}_x(y)} N_x}\Bigr);
\end{equation}
notice that $\hat{\gamma}_{\hat{\sigma}_x(y)} N_x$ is the expected value of $N_{x,y}$ under the null hypothesis. This is the LRS in the $X\rightarrow Y$ direction, which we indicate with the subscript $X\rightarrow Y$. The LRS in the reverse direction, denoted $G_{Y\rightarrow X}^2$, is computed in the same way, after swapping the roles of $X$ and $Y$.


It is well known that $G^2$ is asymptotically $\chi^2$-distributed with df $=(|\mathcal{X}|-1)(|\mathcal{Y}|-1)$ \textit{degrees of freedom}, yielding the ${\tt p}$-value 
\[
{\tt p} = \mathbb{P}[ \chi^2_{\mbox{df}} \geq G^2] = 1-\mathbb{P}[ \chi^2_{\mbox{df}} < G^2],
\]
where $\mathbb{P}[ \chi^2_{\mbox{df}} < G^2]$ is the cumulative distribution function of a $\chi_{\mbox{df}}^2$ distribution. If ${\tt p}$ is less than some significance level (\textit{i.e.}, the test statistic $G^2$ is too large), the null hypothesis is rejected.

Let ${\tt p}^{X\rightarrow Y}$ and ${\tt p}^{Y\rightarrow X}$ be the ${\tt p}-$values of the LRS for testing the uniformity of the channels in both directions, and let $\alpha$ be a significance level for the test \citep{Agresti, Read}, \textit{i.e.}, the null hypothesis is rejected if the ${\tt p}-$value is less than $\alpha$. Having a statistical test of whether an estimated conditional pmf corresponds to a UCM, we adopt a procedure similar to the one proposed by \citet{anm2011}.
    \begin{itemize}
        \item If ${\tt p}^{X\rightarrow Y} \geq \alpha$ and  ${\tt p}^{Y\rightarrow X} < \alpha$, declare $X\rightarrow Y$.
    \item If ${\tt p}^{X\rightarrow Y} < \alpha$ and  ${\tt p}^{Y\rightarrow X} \geq \alpha$, declare $Y\rightarrow X$.
     \item If ${\tt p}^{X\rightarrow Y} < \alpha$ and  ${\tt p}^{Y\rightarrow X} < \alpha$, declare "undecided: wrong model".
     \item If ${\tt p}^{X\rightarrow Y} \geq \alpha$ and  ${\tt p}^{Y\rightarrow X} \geq \alpha$, declare "undecided: both directions possible".
     \end{itemize}

The fourth case (\textit{i.e.}, the hypotheses that the channel is uniform in both directions cannot be rejected) is asymptotically improbable, unless $X$ and $Y$ are independent, due to the identifiability guarantee. Alternatively, to force the method to make a decision between the two causal directions, one may simply decide for $X\rightarrow Y$, if ${\tt p}^{X\rightarrow Y} > {\tt p}^{Y\rightarrow X}$, and for $Y\rightarrow X$, otherwise. 



%% file: resul.tex
\section{Results}
\label{sec:resul}
We compare the proposed approach, on synthetic, benchmark, and real data, with two state-of-the-art methods for categorical variables, for which code is publicly available: DC \citep{dc} ({\small \url{eda.mmci.uni-saarland.de/prj/cisc/}}, with $\epsilon = 0$) and HCR \citep{hcr} (a Python version of the $R$ code available at {\small \url{cran.r-project.org/web/packages/HCR/index.html}}). The code for all the experiments will be made available upon acceptance of the manuscript. 


In the ML estimates underlying our approach (namely \eqref{estimate_thetaxy} and \eqref{approximategamma}), to avoid the problem of zero or vanishing probabilities, we use a small amount ($10^{-3}$) of additive (a.k.a. Dirichlet) smoothing.

\subsection{Identifying the UCM Direction}
The first set of experiments is a sanity check, assessing the ability of the proposed criterion to identify the UCM direction, using synthetic data, with different sample sizes $N$ 
and different sizes of the support sets, $|\mathcal{X}|$, and $|\mathcal{Y}|$.
For each pair $(|\mathcal{X}|, |\mathcal{Y}|)$ and each $N$, we generate $100$ independent datasets using randomly generated UCMs in the $X\rightarrow Y$ direction and the results reported for each $N$ are the corresponding averages. The decision rule is simply to choose $X\rightarrow Y$, if ${\tt p}^{X\rightarrow Y} \geq {\tt p}^{Y\rightarrow X}$ (equivalently, $G_{X\rightarrow Y}^2 \leq G_{Y\rightarrow X}^2$), and $Y\rightarrow X$ (which is wrong), otherwise. The results in Fig. \ref{fig:prob} show that the accuracy achieves high values, close to 100\%, for $N> 500\sim 1000$, without a clear effect of the sizes of the support sets or difference between the non-cyclic and the cyclic cases.

\begin{figure*}[ht]
\center
\includegraphics[scale=0.49]{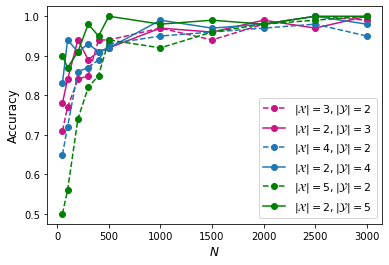}
\qquad 
\includegraphics[scale=0.49]{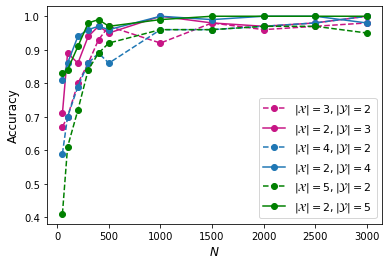}
\vspace{-0.1cm}
\caption{Accuracy in selecting the UCM direction, using the proposed criterion, for different sample and support sizes. Left plot: general UCMs; right plot: cyclic UCMs (CUCM).}
\label{fig:prob}
\end{figure*}

\subsection{Benchmark Data}
We use the 112 pairs in the \textit{cause-effect pairs} benchmark set \citep{causeeffectwebsite} where both variables are categorical and that have as ground truth that either $X \rightarrow Y$ or $Y \rightarrow X$. We set $\alpha = 0.05$ and compare UCM with the two methods mentioned above: DC \citep{dc} and HCR \citep{hcr}. Decisions of ``undecided'' are counted as wrong. The average accuracies of UCM, DC, and HCR reported in Table \ref{tab:results_bench} shows that UCM outperforms both HCR and DC on this dataset. Notice that a random decision would yield accuracy equal to 1/3.

\begin{table*}[hbt]
\centering
     \caption{Average accuracy results on the 112 pairs of the benchmark dataset.}
\vspace{0.5cm}
     \label{tab:results_bench}
     \makebox[\linewidth]{
     \begin{tabular}{| c | c |  c |}
     \hline
     UCM & DC & HCR \\
     \hline\hline
     0.61 & 0.41 & 0. 47\\
     \hline
\end{tabular} }
\end{table*}

\subsection{Real Data}\label{sec:real_data}
Finally, we evaluate the UCM method (again with $\alpha = 0.05$) on real data from the \textit{UCI Machine Learning Repository} \citep{uciwebsite}. We use pairs of variables from the following datasets: Adult, Pittsburgh Bridges, Accute Inflamation, Temperature, and Horse Colic. The datasets and selected pairs, as well as the criteria used to decide what is the ground truth causal direction, are described in Appendix \ref{app_data}. We include only pairs for which a test of independence, at significance level 0.05 \citep{Agresti}, rejects the null hypothesis of independence. Furthermore, we include only pairs for which at least one of the three tested methods chooses one of the causal directions. 
Table \ref{tab:results} shows that the UCM and HCR approaches found the ``correct" causal direction in 5 out of 9 pairs, and DC in 4 pairs. UCM returned only correct decisions or abstained from deciding. This small number of experiments does now allow for reaching any strong conclusions but suggests that UCM performs on par, arguably somewhat better, with DC and HCR.

\begin{table*}[hbt]
\centering
\normalsize
{
     \caption{Results on real data. Wrong decisions are shown in red; UWM stands for "undecided: wrong model". Month is a cyclic variable, thus a CUC was used in the $Y\rightarrow X$ direction.}
\vspace{-0.2cm}
     \label{tab:results}
     \makebox[\linewidth]{
     \begin{tabular}{ c | c  c  c  c c}
     Dataset & $X$ & $Y$ & UCM & DC & HCR\\[0.1cm]
     \hline
     Adult & Occupation & Income & UWM & $X
\rightarrow Y$ &$X \rightarrow Y$ \\
     Adult & Work Class & Income & UWM & $X
\rightarrow Y$   &$X \rightarrow Y$    \\
      Acute Inflammation & Inflam. of urinary bladder & Lumbar pain & $Y
\rightarrow X$ & Inconcl. &  Inconcl.\\
      Acute Inflammation & Inflam. of urinary bladder & Nausea & $Y
\rightarrow X$ & Inconcl. &  Inconcl.\\
 Acute Inflammation & Inflam. of urinary bladder & Burning urethra & $Y
\rightarrow X$ & Inconcl. &  Inconcl.\\
     Pittsburgh Bridges & Material & Lanes & $X \rightarrow Y$ &
\color{red} $Y \rightarrow X$ & $X \rightarrow Y$ \\
     Pittsburgh Bridges & Purpose & Type & UWM &
\color{red} $Y \rightarrow X$ & \color{black} $X \rightarrow Y$ \\
     Temperature & Month & Temperature & $X \rightarrow Y$ &
$X \rightarrow Y$ & \color{red} $Y \rightarrow X$ \\
     Horse Colic & Abdomen Status & Surgical Lesion & UWM &
$X \rightarrow Y$ & $X \rightarrow Y$ \\
\hline 
     \end{tabular} }}
\end{table*}

%% file: concl.tex
\vspace{-0.3cm}
\section{Conclusions}
\label{sec:concl}
We introduced the \textit{uniform channel model} (UCM) to address the cause-effect problem with categorical variables. The proposed approach is based on viewing conditional distributions as communication channels. The UCM can be seen as an ANM-type instantiation of the principle of \textit{independence of cause and mechanism}, preserving a key feature of ANM for quantitative data: the conditional entropy (uncertainty) of the effect given the cause is independent of the cause. The core results of this paper are a proof of the identifiability of the UCM and a proof of its equivalence to a structural causal model with an exogenous variable of fixed cardinality.

To instantiate the approach on finite data, we used classical statistical tests to decide in which of the two directions (if any) the conditional distribution is close enough to correspond to a UCM. The experimental results confirmed the adequacy of the proposed method. By comparing our method with two other recent methods (DC, by \citet{dc}, and HCR, by \citet{hcr}), we found that UCM outperforms those other methods on benchmark datasets and performs on par with those methods on real data.  

As future work, we will aim to extend the proposed method to handle more than two variables. For example, closeness to a uniform channel of the conditional distribution of each variable given its parents can be used in a score-based method. Another direction of research will look at  cases where one variable is categorical and the other is continuous.  In fact, we envision a generalization that subsumes both ANM and the proposed UCM as follows. In the correct causal direction, the conditional distribution of the effect should be equivariant under some transformation group (the element of which is selected by the cause) that is relevant to its domain: shifts, in the ANM case, permutations, for categorical variables, and cyclic permutations for variables with cyclic structure. If this equivariance holds in the causal direction but not in the reverse one, we have identifiability.

\section*{Acknowledgments}
We thank Afonso Bandeira, André Gomes, Francisco Andrade, João Xavier, and José Mourão, for fruitful conversations and important (some of them instrumental) suggestions. We also thank the anonymous reviewers for some important suggestions.

%% file: appendix_a.tex
\section{Proof of Theorem \ref{identifiability}}
\label{app_proof}
{\noindent\bf Proof}: Let us denote $\boldsymbol{\theta}^X = \boldsymbol{\beta}  \in \Delta_{|\mathcal{X}|-1}$ and recall that under the UC assumption, matrix $\boldsymbol \theta^{X\rightarrow Y}$ has the form
\[
\boldsymbol \theta^{X\rightarrow Y} = \begin{bmatrix}
\gamma_{\sigma_1(1)} & \gamma_{\sigma_1(2)} & \cdots & \gamma_{\sigma_1(|\mathcal{Y}|)}\\
\gamma_{\sigma_2(1)} & \gamma_{\sigma_2(2)} & \cdots & \gamma_{\sigma_2(|\mathcal{Y}|)} \\
\vdots & \vdots & \ddots & \vdots \\
\gamma_{\sigma_{|\mathcal{X}|}(1)} & \gamma_{\sigma_{|\mathcal{X}|}(2)} & \cdots & \gamma_{\sigma_{|\mathcal{X}|}(|\mathcal{Y}|)}
\end{bmatrix},
\]
where $\sigma_1,...,\sigma_{|\mathcal{X}|}\in \mathbb{S}_{|\mathcal{Y}|}$ are permutations and 
$\boldsymbol{\gamma} = (\gamma_1, ...., \gamma_{|\mathcal{Y}|} )\in \Delta_{|\mathcal{Y}|-1}$. The assumption that the rows of this matrix are not all equal to each other precludes the two following condition from holding: $\sigma_1 = \sigma_2 = \cdots = \sigma_{|\mathcal{X}|}$ and  $\boldsymbol{\gamma} = (1,1,...,1)/|\mathcal{Y}|$.

Using Bayes' law, it is trivial to obtain the reverse channel, the elements of which are given by
\begin{equation}
{\theta}_{y,x}^{Y\rightarrow X} = p_{X|Y}(x|y) = \frac{p_{Y|X}(y|x)\, p_X(x)}{p_{Y}(y)} = \frac{ \gamma_{\sigma_x(y)} \, \beta_x }{ \sum_{x'\in\mathcal{X}} \gamma_{\sigma_{x'}(y)} \, \beta_{x'}  } =  \frac{a_{y,x}}{A_y} , \label{eq:reverse_theta}
\end{equation}
where $a_{y,x} = \beta_x \gamma_{\sigma_x(y)} $ and $A_y = p_Y(y) \neq 0$ (by assumption).
As in the binary example, using variables $Y$ and $X$ to index rows and columns, respectively, $\boldsymbol{\theta}^{Y\rightarrow X}$ is a row-stochastic matrix:
\[
\boldsymbol{\theta}^{Y\rightarrow X} = \begin{bmatrix} 
a_{1,1}/A_1  & \cdots & a_{1,|\mathcal{X}|}/A_1\\
\vdots & \ddots & \vdots \\
a_{|\mathcal{Y}|,1}/A_{|\mathcal{Y}|} & \cdots & a_{|\mathcal{Y}|,|\mathcal{X}|}/A_{|\mathcal{Y}|}
\end{bmatrix}.
\]
For $\boldsymbol{\theta}^{Y\rightarrow X}$ to correspond to a UC, its rows must be permutations of each other, which is equivalent to all being permutations of one of them, say the first, without loss of generality. We exclude the case where these permutations are all equal to identity, since that would correspond to all rows of $\boldsymbol{\theta}^{Y\rightarrow X}$ being equal to each other, \textit{i.e.}, $X \perp\!\!\!\perp Y$,  which is excluded in the conditions of the theorem. The condition that all the rows are permutations of the first one can be written formally as
\begin{equation}
\exists (\rho_2, ..., \rho_{|\mathcal{Y}|}) \in \mathbb{L}   : \forall y\in \mathcal{Y}\setminus \{1\}, \,\forall x\in\mathcal{X}, \; a_{1,x} / A_1 = a_{y,\rho_y(x)} / A_y , \label{eq:condition}
\end{equation}
where $\mathbb{L} = (\mathbb{S}_{|\mathcal{X}|})^{|\mathcal{Y}|-1} \setminus \mathbb{I}$, with $\mathbb{I} = \left\{\rho_2, ..., \rho_{|\mathcal{Y}|}: \rho_2 = ... =  \rho_{|\mathcal{Y}|} = \iota \right\}$, and $\iota$ is the identity permutation. In words, $\mathbb{L}$ is the set of all $(|\mathcal{Y}|-1)$-tuples of permutations of $|\mathcal{X}|$ elements, except for the one in which all permutations are identity. 

The equality $a_{1,x} / A_1 = a_{y,\rho_y(x)} / A_y$ is equivalent to $(a_{1,x} \, A_y  - a_{y,\rho_y(x)} \, A_1 )^2 = 0$,
thus the following equivalence holds:
\[
\bigl(\forall y\in \mathcal{Y}\setminus \{1\}, \,\forall x\in\mathcal{X}, \; a_{1,x} / A_1 = a_{y,\rho_y(x)} / A_y\bigr) \;\;\; \Leftrightarrow  \;\;\; Q_{\boldsymbol{\rho}}(\boldsymbol{\theta})  = 0,
\]
with 
\begin{equation}
Q_{\boldsymbol{\rho}}(\boldsymbol{\theta}) = \sum_{y\in \mathcal{Y}\setminus \{1\}} \sum_{x\in\mathcal{X}} (a_{1,x} \, A_y  - a_{y,\rho_y(x)} \, A_1 )^2,\label{eq_Qrho}
\end{equation}
where we have written the model parameters compactly as $\boldsymbol{\theta} = (\boldsymbol{\beta} ,\boldsymbol{\gamma}) \in \Delta_{|\mathcal{X}|-1} \times \Delta_{|\mathcal{Y}|-1} \subset \mathbb{R}^{|\mathcal{X}|} \times \mathbb{R}^{|\mathcal{Y}|}$, and denoted $\boldsymbol{\rho} = ( \rho_2, ..., \rho_{|\mathcal{Y}|} ) \in \mathbb{L}$. A key observation is that $Q_{\boldsymbol{\rho}}(\boldsymbol{\theta})$ is a polynomial in the elements of $\boldsymbol{\theta}$, since the $a_{y,x}$ and the $A_y$ are themselves polynomials (either products of two elements or sums of products of pairs of elements) as is clear in \eqref{eq:reverse_theta}.  

Finally, the existential quantifier  in  \eqref{eq:condition} can be re-written using a product, \textit{i.e.},
\[
\bigl( \exists \boldsymbol{\rho} \in \mathbb{L} :\;  Q_{\boldsymbol{\rho}}(\boldsymbol{\theta}) = 0 \bigr) \;\;\;\Leftrightarrow \;\;  R(\boldsymbol{\theta}) = 0, \hspace{0.7cm} \mbox{where}\;\; R(\boldsymbol{\theta}) \; =\!\!  \prod_{\boldsymbol{\rho} \in \mathbb{L} } Q_{\boldsymbol{\rho}}(\boldsymbol{\theta}).
\]
Since $R(\boldsymbol{\theta})$ a product of polynomials, it is itself a polynomial. Consequently, we have shown that the UC condition in \eqref{eq:condition} corresponds to having $\boldsymbol{\theta}$ as a root of a polynomial. 

The rest of the proof relies on a classical result about polynomials \citep{Federer1969}: let $S:\mathbb{R}^n\rightarrow \mathbb{R}$ be a polynomial that is not identically zero; then, the set $S^{-1}(0) = \{{\bf u}\in\mathbb{R}^n: \, S({\bf u})=0\}$ has zero Labesgue measure in $\mathbb{R}^n$. All that is left to show then is that $R(\boldsymbol{\theta})$ is not identically zero. For this purpose, we can ignore the valid parameter space $\Delta_{|\mathcal{X}|-1} \times \Delta_{|\mathcal{Y}|-1}$, because if $R^{-1}(0)$ has zero Lebesgue measure in $\mathbb{R}^{|\mathcal{X}|} \times \mathbb{R}^{|\mathcal{Y}|}$, so does the intersection $R^{-1}(0) \cap (\Delta_{|\mathcal{X}|-1} \times \Delta_{|\mathcal{Y}|-1})$. We can also ignore the condition  $\boldsymbol{\gamma} \neq (1,...,1)/|\mathcal{Y}|$, since this is a single point, thus a set of zero measure. 

A sufficient and necessary condition for $R(\boldsymbol{\theta})$ not to be identically zero is that none of its factors $Q_{\boldsymbol{\rho}}(\boldsymbol{\theta})$ is identically zero\footnote{Recall that a product of two polynomials with real coefficients is identically zero only if at least one of the factors is identically zero. This is a classical result from abstract algebra, which in the language thereof is stated as follows: the ring of all polynomials in $n$ variables with real coefficients is an \textit{integral domain} or \textit{entire ring}, that is, it does not have divisors of zero \citep{Lang}. The result generalizes trivially, by induction, to products of more than two polynomials.}. To show that no $Q_{\boldsymbol{\rho}}(\boldsymbol{\theta})$ is  identically zero, let us write it explicitly, using the definitions of $a_{y,x}$ and $A_y$ in \eqref{eq:reverse_theta}:
\begin{equation}
Q_{\boldsymbol{\rho}}(\boldsymbol{\theta}) = \sum_{y\in \mathcal{Y}\setminus \{1\}} \sum_{x\in\mathcal{X}} \Bigl( \beta_x \gamma_{\sigma_x(1)} \sum_{x'\in \mathcal{X}} \beta_{x'} \gamma_{\sigma_{x'}(y)} - \beta_{\rho_{y}(x)} \gamma_{\sigma_{\rho_y(x)}(y)}\sum_{x'\in \mathcal{X}} \beta_{x'} \gamma_{\sigma_{x'}(1)} \Bigr)^2.\label{eq_Qrho2}
\end{equation}
Since $Q_{\boldsymbol{\rho}}(\boldsymbol{\theta})$ is a sum of non-negative terms, to show that it is not identically zero, it suffices to show that one of the terms in the sum is strictly positive for some choice of $\boldsymbol{\theta}$. The condition $\boldsymbol{\rho} = (\rho_2,...,\rho_{|\mathcal{Y}|}) \in \mathbb{L}$ means that at least one of the permutations  $\rho_2,...,\rho_{|\mathcal{Y}|}$ is not the identity, which implies that there is at least one pair $(x,y)$ such that $\rho_y (x) \neq x$. Let $y$ and $x$ be one such pair. Choosing $\boldsymbol{\gamma} = (1,...,1)$ and $\boldsymbol{\beta} \in \Delta_{|\mathcal{X}|-1}$ such that all components are different from each other ($i\neq j\, \Rightarrow \beta_i \neq \beta_j$), we have (noticing that $\sum_{x'\in\mathcal{X}}\beta_{x'} = 1$)
\begin{align}
Q_{\boldsymbol{\rho}}(\boldsymbol{\theta}) & = \Bigl( \beta_x  \sum_{x'\in \mathcal{X}} \beta_{x'} - \beta_{\rho_{y}(x)} \sum_{x'\in \mathcal{X}} \beta_{x'} \Bigr)^2 +  
\sum_{y' \neq y } \sum_{x'\neq x } (\cdots)^2 \label{eq_Qrho3}\\
& = \bigl( \beta_x - \beta_{\rho_{y}(x)} \bigr)^2 +\; \mbox{non-negative terms} > 0.
\end{align}
In conclusion, since none of the $Q_{\boldsymbol{\rho}}(\boldsymbol{\theta})$ polynomials is identically zero, $R(\boldsymbol{\theta})$ is also not identically zero, consequently its zero set has zero Lebesgue measure. 
\hfill $\blacksquare$

%% file: appendix_b.tex
\section{Proof of Proposition \ref{proof_UCestimate}}
\label{proof_estimate}
{\noindent\bf Proof}: Noticing that the permutations that map $\boldsymbol{\gamma}$ to each row of $\boldsymbol{\theta}^{Y|X}$ are arbitrary, there is no loss of generality in assuming $\gamma_1 \geq \gamma_2 \geq \cdots \geq \gamma_{|\mathcal{Y}|}$, \textit{i.e.}, $\boldsymbol{\gamma}\in \mathcal{K}_{|\mathcal{Y}|}$, the so-called \textit{monotone cone} \citep{Best1990}. 
Furthermore, it is more convenient to formulate the problem w.r.t. the inverse permutations, denoted as $\tau_1,\ldots, \tau_{|\mathcal{X}|} \in \mathbb{S}_{|\mathcal{Y}|}$. The problem can thus be written as
\begin{align}
    \hat{\boldsymbol \gamma}, \hat{\tau}_1, \ldots, \hat{\tau}_{|\mathcal{X}|} \;  = \hspace{-0.5cm} \underset{ \begin{array}{c} \boldsymbol \gamma \in (\Delta _{|\mathcal{Y}|-1} \cap \mathcal{K}_{|\mathcal{Y}|}) \\ 
    \tau_1,...,\tau_{|\mathcal{X}|} \in \mathbb{S}_{|\mathcal{Y}|}\end{array} }{\text{arg max}} \hspace{-0.5cm} \mathcal{L}(\boldsymbol \gamma, \tau_1,...,\tau_{|\mathcal{X}|} ), \label{thirdscenarioc}
\end{align}
where 
\begin{align}
    \mathcal{L}(\boldsymbol \gamma, \tau_1,...,\tau_{|\mathcal{X}|} ) = \sum_{x\in \mathcal{X}} \sum_{y\in \mathcal{Y}}  N_{x,\tau_x (y)}\log \gamma_y. \label{thirdscenario_bc}
\end{align}
 
The assumption $\gamma_1 \geq \gamma_2 \geq \cdots \geq \gamma_{|\mathcal{Y}|}$ makes the  maximization w.r.t. $\tau_1,...,\tau_{|\mathcal{X}|}$ independent of the particular values of $\boldsymbol{\gamma}$ as well as separable into a collection of $|\mathcal{X}|$ independent maximizations, 
\begin{equation}
    \hat{\tau}_x = \underset{ \tau_x \in  \mathbb{S}_{|\mathcal{Y}|}}{\text{arg max}} \sum_{y\in \mathcal{Y}}  N_{x,\tau_x (y)}\log \gamma_y,\label{thirdscenario2}
\end{equation}
for $x\in\mathcal{X}$. Solving \eqref{thirdscenario2} is a simple application of the \textit{rearrangement inequality}\footnote{Given any two non-decreasing sequences of $n$ real numbers, $x_1 \leq \ldots \leq x_n$ and $y_1 \leq \ldots \leq y_n$,  
\[
\forall\sigma \in \mathbb{S}_n, \;\; \sum\limits_{i=1}^n x_{\sigma_{(i)}} y_i \leq \sum\limits_{i=1}^n x_i y_i.
\]} \citep{rearrangementinequality}. Since $\log \gamma_1 \geq \log \gamma_2 \geq \cdots \geq \log \gamma_{|\mathcal{Y}|}$, the solution $\hat{\tau}_x$ is any permutation that also sorts $\{N_{x,1},..., N_{x,|\mathcal{Y}|}\}$ into non-increasing order:
\begin{equation}
\hat{\tau}_x\;\; \mbox{is such that}\;\; N_{x, \hat{\tau}_x (1)} \geq \cdots \geq N_{x, \hat{\tau}_x (|\mathcal{Y}|)}.\label{sortedNx}
\end{equation}
If all the elements of $\{N_{x,1},..., N_{x,|\mathcal{Y}|}\}$ are different, the optimal permutation is unique; otherwise, there are several optimal permutations, all achieving the same maximum. The cost of finding $\hat{\tau}_1, \ldots, \hat{\tau}_{|\mathcal{X}|}$ is $O(|\mathcal{X}|\, |\mathcal{Y}|\, \log |\mathcal{Y}|)$, since it requires $|\mathcal{X}|$ sorting operations, each with $|\mathcal{Y}|$ elements. 

Plugging $\hat{\tau}_1, \ldots, \hat{\tau}_{|\mathcal{X}|}$ back into \eqref{thirdscenario}--\eqref{thirdscenario_b} and swapping the summation order, yields 
\begin{align}
    \hat{\boldsymbol{\gamma}} \;\;  = \!\! \underset{\boldsymbol \gamma \in (\Delta _{|\mathcal{Y}|-1} \cap \mathcal{K}_{|\mathcal{Y}|}) }{\text{arg max}} \sum_{y\in \mathcal{Y}}  \log \gamma_y \sum_{x\in \mathcal{X}}  N_{x,\hat\tau_x (y)}. \label{thirdscenario3}
\end{align}
This problem is the same as \eqref{eq:UDC_known}, with $\hat\tau_x$ in the place of $\tau_x$ and with the additional constraint $\boldsymbol{\gamma} \in \mathcal{K}_{|\mathcal{Y}|}$. Temporarily ignoring this constraint leads to (see \eqref{approximategamma}), 
\begin{equation}
    {\displaystyle \hat{\gamma}_y = \frac{1}{N}\sum\limits_{x\in \mathcal{X}} N_{x, \hat\tau_x (y)}}, \;\; \; \mbox{for $y \in \mathcal{Y}$}.
\label{approximategamma2}
\end{equation}
The fact that $N_{x, \hat{\tau}_x (1)} \geq N_{x, \hat{\tau}_x (2)} \geq \cdots \geq N_{x, \hat{\tau}_x (|\mathcal{Y}|)}$, for any $x\in\mathcal{X}$, implies that $\hat{\boldsymbol{\gamma}} \in \mathcal{K}_{|\mathcal{Y}|}$, without having to include this constraint.  Consequently, problem \eqref{thirdscenario}--\eqref{thirdscenario_b}  has a global solution given by \eqref{approximategamma2} and \eqref{sortedNx}.
\hfill $\blacksquare$